\documentclass[10pt,a4paper]{article}

\newcommand{\shortversion}[1]{}
\newcommand{\longversion}[1]{#1}

\longversion{\usepackage[hscale=0.7,scale=0.75]{geometry}}

\shortversion{\usepackage{ijcai11}}

\usepackage{tikz}

\shortversion{\usepackage{times}}

\longversion{\let\paragraph=\subsection}

\usepackage{amsmath}
\usepackage{amssymb}
\usepackage{amsthm}

\shortversion{\usepackage[belowskip=-6pt,aboveskip=5pt]{caption}}

\usepackage{tikz}
\usetikzlibrary{shapes,shapes.callouts,shadows,arrows,backgrounds,%
matrix,patterns,arrows,decorations.pathmorphing,decorations.pathreplacing,%
positioning,fit,calc,decorations.text%
}

\newtheorem{LEM}{Lemma} 
\newtheorem{THE}{Theorem} 
\newtheorem{COR}{Corollary}

\newtheorem{PRO}{Proposition} 
 
\theoremstyle{remark}
\newtheorem{EXA}{Example}

\def\hy{\hbox{-}\nobreak\hskip0pt} 

\newcommand{\SB}{\{\,}%
\newcommand{\SM}{\;{|}\;}%
\newcommand{\SE}{\,\}}%

\newcommand{\Card}[1]{|#1|}
\let\phi=\varphi
\let\epsilon=\varepsilon

\newcommand{\CCC}{\mathcal{C}}

\renewcommand{\P}{\text{\normalfont P}}
\newcommand{\NP}{\text{\normalfont NP}}

\newcommand{\PIC}{\Pi^\P_2}
\newcommand{\SIGC}{\Sigma^\P_2}
\newcommand{\coNP}{\text{\normalfont coNP}}

\newcommand{\IN}{\text{\scshape in}}
\newcommand{\OUT}{\text{\scshape out}}
\newcommand{\UND}{\text{\scshape und}}
\newcommand{\UD}{\text{\scshape ud}}
\newcommand{\DEF}{\text{\scshape def}}

\newcommand{\CA}{\text{\normalfont CA}}
\newcommand{\SA}{\text{\normalfont SA}}

\newcommand{\ACYC}{\text{\scshape Acyc}}
\newcommand{\BIP}{\text{\scshape Bip}}

\newcommand{\NOEVEN}{\text{\scshape Noeven}}

\newcommand{\SYM}{\text{\scshape Sym}}

\newcommand{\STB}{\text{\normalfont stb}}
\newcommand{\ADM}{\text{\normalfont adm}}
\newcommand{\PRF}{\text{\normalfont prf}}
\newcommand{\COM}{\text{\normalfont com}}

\newcommand{\SEM}{\text{\normalfont sem}}
\newcommand{\bsdist}{\text{\normalfont dist}}

\newcommand{\lext}[1]{#1^*}
\newcommand{\labFS}[2]{\textup{lab}(#1,#2)}

\newcommand{\citex}[1]{\citeauthor{#1}~\shortcite{#1}}

\newcommand{\captionfonts}{\small}
\makeatletter

\longversion{
  \def\leftcite{\@up[}\def\rightcite{\@up]}
  
  \def\cite{\def\citeauthoryear##1##2{\def\@thisauthor{##1}%
               \ifx \@lastauthor \@thisauthor \relax \else##1, \fi ##2}\@icite}
  \def\shortcite{\def\citeauthoryear##1##2{##2}\@icite}
  
  \def\citeauthor{\def\citeauthoryear##1##2{##1}\@nbcite}
  \def\citeyear{\def\citeauthoryear##1##2{##2}\@nbcite}
  
  \def\@icite{\leavevmode\def\@citeseppen{-1000}%
   \def\@cite##1##2{\leftcite\nobreak\hskip 0in{##1\if@tempswa , ##2\fi}\rightcite}%
   \@ifnextchar [{\@tempswatrue\@citex}{\@tempswafalse\@citex[]}}
  \def\@nbcite{\leavevmode\def\@citeseppen{1000}%
   \def\@cite##1##2{{##1\if@tempswa , ##2\fi}}%
   \@ifnextchar [{\@tempswatrue\@citex}{\@tempswafalse\@citex[]}}
  
  \def\@citex[#1]#2{%
    \def\@lastauthor{}\def\@citea{}%
    \@cite{\@for\@citeb:=#2\do
      {\@citea\def\@citea{;\penalty\@citeseppen\ }%
       \if@filesw\immediate\write\@auxout{\string\citation{\@citeb}}\fi
       \@ifundefined{b@\@citeb}{\def\@thisauthor{}{\bf ?}\@warning
         {Citation `\@citeb' on page \thepage \space undefined}}%
       {\csname b@\@citeb\endcsname}\let\@lastauthor\@thisauthor}}{#1}}
  
  \def\@biblabel#1{\def\citeauthoryear##1##2{##1, ##2}\@up{[}#1\@up{]}\hfill}
  
  \def\@up#1{\leavevmode\raise.2ex\hbox{#1}}
}

\shortversion{
\long\def\@makecaption#1#2{%
 \vskip\abovecaptionskip
 \sbox\@tempboxa{{\captionfonts #1: #2}}%
 \ifdim \wd\@tempboxa >\hsize
   {\captionfonts #1: #2\par}
 \else
   \hbox to\hsize{\hfil\box\@tempboxa\hfil}%
 \fi
 \vskip\belowcaptionskip}

 \def\section{\@startsection{section}{1}{\z@}{-5pt plus
     -3pt minus -2pt}{3pt plus 2pt minus 1pt}{\Large\bf\raggedright}}
 \def\subsection{\@startsection{subsection}{2}{\z@}{-5pt plus
     -2pt minus -2pt}{2pt plus 2pt minus 1pt}{\large\bf\raggedright}}
 \def\subsubsection{\@startsection{subparagraph}{3}{\z@}{-5pt plus
    -2pt minus -1pt}{1pt plus 1pt minus 1pt}{\normalsize\bf\raggedright}}
 \def\paragraph{\@startsection{paragraph}{4}{\z@}%
                                     {5pt \@plus1ex \@minus.2ex}%
                                     {-1em}%
                                     {\normalfont\normalsize\bfseries}}

\def\thm@space@setup{%
  \thm@preskip=4pt \thm@postskip=\thm@preskip}
\renewenvironment{proof}[1][\proofname]{%
  \par  \setlength{\topsep}{0pt}  \pushQED{\qed}%
  \normalfont 
  \trivlist  \item[\hskip\labelsep        \itshape
  #1\@addpunct{.}]\ignorespaces}{%
  \popQED\endtrivlist\@endpefalse}
}
\makeatother

\title{Augmenting Tractable Fragments of Abstract
  Argumentation\thanks{Research funded by the ERC (COMPLEX REASON,
    239962).}}

\shortversion{\author{Sebastian Ordyniak \and Stefan Szeider\\
  Vienna University of Technology\\
  Vienna,  Austria\\
  ordyniak@kr.tuwien.ac.at, stefan@szeider.net}}

\longversion{\author{Sebastian Ordyniak and Stefan Szeider\\[4pt]
  Vienna University of Technology,
  Vienna,  Austria\\
    ordyniak@kr.tuwien.ac.at, stefan@szeider.net}}

\date{}

\begin{document}

\maketitle

\begin{abstract}
  We present a new and compelling approach to the efficient solution of
  important computational problems that arise in the context of abstract
  argumentation.  Our approach makes known algorithms defined for
  restricted fragments generally applicable, at a computational cost
  that scales with the distance from the fragment. Thus, in a certain
  sense, we gradually augment  tractable fragments.  Surprisingly, it
  turns out that some tractable fragments admit such an augmentation
  and that others do not.
  
  More specifically, we show that the problems of credulous and
  skeptical acceptance are fixed-parameter tractable when parameterized
  by the distance from the fragment of acyclic argumentation frameworks.
  Other tractable fragments such as the fragments of symmetrical and
  bipartite frameworks seem to prohibit an augmentation: the acceptance
  problems are already intractable for frameworks at distance 1 from the
  fragments.

  For our study we use a broad setting and consider several different
  semantics.  For the algorithmic results we utilize recent advances in
  fixed-parameter tractability.
%
\end{abstract}

\section{Introduction}

The study of arguments as abstract entities and their interaction in
form of \emph{attacks} as introduced\\ by~\citex{Dung95} has become one of
the most active research branches within Artificial Intelligence, Logic
and Reasoning~\cite{BenchcaponDunne07,BesnardHunter08,RahwanSimari09}.
Abstract argumentation provides suitable concepts and formalisms to
study, represent, and process various reasoning problems most
prominently in defeasible reasoning (see,
e.g.,~\cite{Pollock92,BondarenkoDungKowalskiToni97}) and agent
interaction (see, e.g., \cite{ParsonsWooldridgeAmgoud03}).

A main issue for any argumentation system is the selection of acceptable
sets of arguments, called extensions. However, important computational
problems such as determining whether an argument belongs to some
extension (Credulous Acceptance) or all extensions (Skeptical
Acceptance), are intractable (see, e.g.,
\cite{DimopoulosTorres96,DunneBenchcapon02}). The significance of
efficient algorithms for these problems is evident. However, a few
tractable fragments are known where the acceptance problems can be
efficiently solved: the fragments of acyclic~\cite{Dung95}, symmetric
\cite{CostemarquisDevredMarquis05}, bipartite~\cite{Dunne07}, and noeven
\cite{DunneBenchcapon02} argumentation frameworks.

It seems unlikely that an argumentation framework originating from a
real-world application belongs to one of the known tractable fragments,
but it might be ``close'' to a tractable fragment. 
In this paper we study the natural and significant question of whether
we can solve the relevant problems efficiently for argumentation
frameworks that are of small distance to a tractable fragment. One would
certainly have to pay some extra computational cost that increases with
the distance from the tractable fragment, but ideally this extra cost
should scale gradually with the distance.

\paragraph{Results} We show that the fragments of \emph{acyclic} and
\emph{noeven} argumentation frameworks admit an augmentation.  In
particular, we show that we can solve Credulous and Skeptical Acceptance
in polynomial time for argumentation frameworks that are of bounded
distance from either of the two fragments.  We further show that with
respect to the acyclic fragment, the order of the polynomial time bound
is independent of the distance, which means that both acceptance
problems are \emph{fixed-parameter tractable} (see
\cite{DowneyFellows99}) when parameterized by the distance from the
acyclic fragment.

In way of contrast, we show that the fragments of \emph{bipartite} and
\emph{symmetric} argumentation frameworks do not admit an
augmentation. In particular, we show that the problems Credulous and
Skeptical Acceptance are already \emph{intractable} (i.e., (co)NP-hard)
for argumentation frameworks at distance~$1$ from either of the two
fragments.

We further show that the parameter ``distance to the fragment of acyclic
frameworks'' is \emph{incomparable} with previously considered
parameters that also admit fixed-parameter tractable argumentation
\cite{Dunne07,DvorakSzeiderWoltran10}.  Hence our approach provides an
efficient solution for instances that are hard for known methods.

To get a broad picture of the complexity landscape we take several
popular semantics into consideration, namely admissible, preferred,
complete, semi-stable and stable semantics
(see \cite{BaroniGiacomin09}).

Our approach is inspired by the notion of ``backdoors'' which are
frequently used in the area of propositional satisfiability (see, e.g.,
\cite{WilliamsGomesSelman03,GottlobSzeider08,SamerSzeider08c}), and also
for quantified Boolean formulas and nonmonotonic reasoning
\cite{SamerSzeider09a,FichteSzeider11}.

\section{Preliminaries}

An \emph{abstract argumentation system} or \emph{argumentation
  framework} (\emph{AF}, for short) is a pair $(X,A)$ where $X$ is a
finite set of elements called \emph{arguments} and $A\subseteq X\times
X$ is a binary relation called \emph{attack relation}. If $(x,y)\in A$
we say that \emph{$x$ attacks $y$} and that $x$ is an \emph{attacker}
of~$y$.

An AF $F=(X,A)$ can be considered as a directed graph, and therefore it
is convenient to borrow notions and notation from graph theory.  For a
set of arguments $Y \subseteq X$ we denote by $F[Y]$ the AF $(Y,\SB
(x,y) \in A \SM x,y \in Y \SE)$ and by $F - Y$ the AF $F[X \setminus
Y]$. 

\begin{EXA}\label{exa:af}
  An AF with arguments $1,\dots,5$ and attacks 
  $(1,2)$, 
  $(1,4)$,   
  $(2,1)$,
  $(2,3)$, 
  $(2,5)$, 
  $(3,2)$, 
  $(3,4)$, 
  $(4,1)$, 
  $(4,2)$, 
  $(4,3)$, 
  $(5,4)$
  is displayed in Fig.~\ref{fig:ex-af}. \qed
\end{EXA}
Let $F=(X,A)$ be an AF, $S \subseteq X$ and $x \in X$. We say that $x$
is \emph{defended} (in $F$) by $S$ if for each $x' \in X$ such that
$(x',x) \in A$ there is an $x'' \in S$ such that $(x'',x') \in A$. We
denote by $S_F^+$ the set of arguments $x \in X$ such that either $x \in
S$ or there is an $x' \in S$ with $(x',x) \in A$, and we omit the
subscript if $F$ is clear from the context. We say $S$ is
\emph{conflict-free} if there are no arguments $x,x' \in S$ with $(x,x')
\in A$.

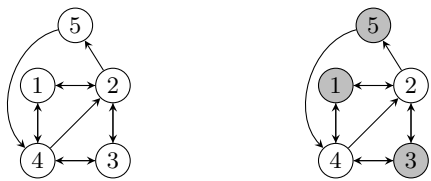
\begin{figure}[bh]
   \centering\small
    \begin{tikzpicture}[xscale=0.5,yscale=0.5,>=stealth]
      \tikzstyle{every node}=[circle,draw,minimum size=4mm,inner sep=2pt]

      \begin{scope}[scale=1]
        \draw
        (0,0) node (p1) {$1$}
        (2,0) node (p2) {$2$}
        (2,-2) node (p3) {$3$}
        (0,-2) node (p4) {$4$}
        (1,1.6) node (p5) {$5$}
        ;
        
        \draw
        (p1) edge[->] (p2)
        (p2) edge[->] (p1)
        
        (p2) edge[->] (p3)
        (p3) edge[->] (p2)
        
        (p3) edge[->] (p4)
        (p4) edge[->] (p3)
        
        (p1) edge[->] (p4)
        (p4) edge[->] (p1)
        
        (p2) edge[->] (p5)
        (p5) edge[->,bend right=60] (p4)
        
        (p4) edge[->] (p2)
        ;
      \end{scope}
    \end{tikzpicture}\hspace{20mm}
    \begin{tikzpicture}[xscale=0.5,yscale=0.5,>=stealth]
      \tikzstyle{every node}=[circle,draw,minimum size=4mm,inner sep=2pt]

      \begin{scope}[scale=1]
        \draw
        (0,0) node[fill=lightgray] (p1) {$1$}
        (2,0) node (p2) {$2$}
        (2,-2) node[fill=lightgray] (p3) {$3$}
        (0,-2) node (p4) {$4$}
        (1,1.6) node[fill=lightgray] (p5) {$5$}
        ;
        
        \draw
        (p1) edge[->] (p2)
        (p2) edge[->] (p1)
        
        (p2) edge[->] (p3)
        (p3) edge[->] (p2)
        
        (p3) edge[->] (p4)
        (p4) edge[->] (p3)
        
        (p1) edge[->] (p4)
        (p4) edge[->] (p1)
        
        (p2) edge[->] (p5)
        (p5) edge[->,bend right=60] (p4)
        
        (p4) edge[->] (p2)
        ;
      \end{scope}
    \end{tikzpicture}
 
    \caption{Left: the AF $F$ from Example~\ref{exa:af}. Right:
      indicated in gray the
      only non-empty complete extension of $F$.}
  \label{fig:ex-af}
\end{figure}

Next we define commonly used semantics of AFs, see the survey of
\citex{BaroniGiacomin09}.  We consider a semantics $\sigma$ as a mapping
that assigns to each AF $F=(X,A)$ a family $\sigma(F) \subseteq 2^X$ of
sets of arguments, called \emph{extensions}.  We denote by \ADM{},
\PRF{}, \COM{}, \SEM{} and \STB{} the \emph{admissible},
\emph{preferred}, \emph{complete}, \emph{semi-stable} and \emph{stable}
semantics, respectively.  These five semantics are characterized by the
following conditions which hold for each AF $F=(X,A)$ and each
conflict-free set $S\subseteq X$.
\begin{itemize}
\item $S \in \ADM{}(F)$ if each $s \in S$ is defended by $S$.
\item $S \in \PRF{}(F)$ if $S \in \ADM{}(F)$ and there is no $T \in
  \ADM{}(F)$ with $S \subsetneq T$.
\item $S \in \COM{}(F)$ if $S \in \ADM{}(F)$ and every argument that is
  defended by $S$ is contained in $S$.
\item
  $S \in \SEM{}(F)$ if $S \in \ADM{}(F)$ and there is no $T \in \ADM{}(F)$ with
  $S^+ \subsetneq T^+$.
\item $S \in \STB{}(F)$ if $S^+=X$.
\end{itemize}
Let $F=(X,A)$ be an AF, $x \in X$ and $\sigma \in \{\ADM{}$, $\PRF{}$,
$\COM{}$, $\SEM{}$, $\STB{}\}$.  The argument $x$ is \emph{credulously
  accepted} in $F$ with respect to $\sigma$ if $x$ is contained in some
extension $S \in \sigma(F)$, and $x$ is \emph{skeptically accepted} in
$F$ with respect to $\sigma$ if $x$ is contained in all extensions $S
\in \sigma(F)$.

Each semantics $\sigma$ gives rise to the following two fundamental
computational problems: \textsc{$\sigma$-Credulous Acceptance} and
\textsc{$\sigma$-Skeptical Acceptance}, in symbols $\CA_\sigma$ and
$\SA_\sigma$, respectively.  Both problems take as instance an AF
$F=(X,A)$ together with an argument $x\in X$. Problem $\CA_\sigma$ asks
whether $F$ is credulously accepted in $F$, problem $\SA_\sigma$ asks
whether $F$ is skeptically accepted in $F$. Table~\ref{tab:complexity},
summarizes the complexities of these problems for the considered
semantics~(see \cite{DvorakWoltran10b}).

\begin{table}[t]
  \begin{center}
    \begin{tabular}{c|c|c}
      $\sigma$ & $\CA_\sigma$ & $\SA_\sigma$ \\
      \hline
      \ADM{} & \NP-complete & trivial \\
      \PRF{} & \NP-complete & $\PIC$-complete \\\
      \COM{} & \NP-complete & \P-complete \\
      \STB{} & \NP-complete & \coNP-complete \\
      \SEM{} & $\SIGC$-complete & $\PIC$-complete
    \end{tabular}
  \end{center}
  \caption{Complexity of credulous and skeptical acceptance for various
    semantics $\sigma$.}
\label{tab:complexity}
\end{table}

\begin{EXA}\label{exa:af-ex}
  Consider the AF $F$ from Example~\ref{exa:af} and the complete
  semantics (\COM{}). $F$ has two complete extensions $\emptyset$ and
  $\{1,3,5\}$, see Fig.~\ref{fig:ex-af}. Consequently, the
  arguments $1$, $3$ and $5$ are credulously accepted in $F$ and the
  arguments $2$ and $4$ are not. Furthermore, because of the complete
  extension $\emptyset$, no argument of $F$ is skeptically accepted.
  \qed
\end{EXA}                       

In the following we list classes of AFs for which CA and SA are known to
be solvable in polynomial time
\cite{Dung95,BaroniGiacomin09,CostemarquisDevredMarquis05,Dunne07}.
\begin{itemize}
\item $\ACYC$ is the class of \emph{acyclic} argumentation frameworks, i.e.,
  of AFs that do not contain  directed cycles.
\item $\NOEVEN$ is the class of \emph{noeven} argumentation frameworks,
  i.e., of AFs that do not contain directed cycles of even length.
\item $\SYM$ is the class of \emph{symmetric} argumentation frameworks,
  i.e., of AFs whose attack relation is symmetric.
\item $\BIP$ is the class of \emph{bipartite} argumentation frameworks,
  i.e., of AFs whose sets of arguments can be partitioned into two
  conflict-free sets.
\end{itemize}
\begin{LEM}\label{lem:nice}
  The classes $\ACYC$, $\NOEVEN$, $\SYM$ and $\BIP$ can be recognized in
  polynomial time (i.e., given an AF $F$, we can decide in polynomial
  time whether $F$ belongs to any of the four classes).
\end{LEM}
\begin{proof}
  The statement of the lemma is easily seen for the classes $\ACYC$,
  $\BIP$ and $\SYM$. For class $\NOEVEN$ it follows by a result of 
 \citex{RobertsonSeymourThomas99}.
\end{proof}
Since the recognition and the acceptance problems are polynomial for
these classes, we consider them as ``tractable fragments of abstract
argumentation.''

\paragraph{Parameterized Complexity}
For our investigation we need to take two measurements into account: the
input size $n$ of the given AF $F$ and the distance~$k$ of $F$ from a
tractable fragment. The theory of \emph{parameterized complexity},
introduced and pioneered by \citex{DowneyFellows99}, provides the
adequate concept and tools for such an investigation. We outline the
basic notions of parameterized complexity that are relevant for this
paper, for an in-depth treatment we refer to other sources
\cite{FlumGrohe06,Niedermeier06}.

An instance of a parameterized problem is a pair $(I,k)$ where $I$ is
the \emph{main part} and $k$ is the \emph{parameter}; the latter is
usually a non-negative integer.  A parameterized problem is
\emph{fixed-parameter tractable} (FPT) if there exist a computable
function $f$ such that instances $(I,k)$ of size $n$ can be solved in
time $f(k)\cdot n^{O(1)}$.  Fixed-parameter tractable problems are also called
\emph{uniform polynomial-time tractable} because if $k$ is considered
constant, then instances with parameter $k$ can be solved in polynomial
time where the order of the polynomial is independent of $k$, in
contrast to \emph{non-uniform polynomial-time} running times such as
$n^{O(k)}$.  Thus we have three complexity categories for parameterized
problems: (1) problems that are fixed-parameter tractable (uniform
polynomial-time tractable), (2) problems that are non-uniform
polynomial-time tractable, and (3) problems that are $\NP$-hard or
$\coNP$-hard if the parameter is fixed to some constant (such as $k$-SAT
which is $\NP$-hard for $k=3$).

\paragraph{Backdoors}
We borrow and adapt the concept of backdoors from the area of
propositional satisfiability
\cite{WilliamsGomesSelman03,GottlobSzeider08,SamerSzeider08c}.  Let
$\CCC$ be a class of AFs, $F=(X,A)$ an AF, and $Y \subseteq X$.  We call
$Y$ a \emph{$\CCC$-backdoor for $F$} if $F - Y \in \CCC$.  We write
$\bsdist_\CCC(F)$ for the size of a smallest $\CCC$-backdoor for
$F$, i.e., $\bsdist_\CCC(F)$ represents the distance of $F$ from the
class $\CCC$. For an illustration see Fig.~\ref{fig:ex-af-bs}.

\begin{figure}[tb]
  \small\centering
  \noindent\shortversion{\hspace{-3mm}}\longversion{\hspace{2cm}}
      \begin{tikzpicture}[xscale=0.5,yscale=0.5,>=stealth]
      \tikzstyle{every circle node}=[circle,draw,minimum size=4mm,inner sep=2pt]

      \begin{scope}[scale=1]
        \draw
        (1,-3) node {$\ACYC$}
        (0,0) node[circle] (p1) {$1$}
        (2,0) node[circle,fill=lightgray] (p2) {$2$}
        (2,-2) node[circle] (p3) {$3$}
        (0,-2) node[circle,fill=lightgray] (p4) {$4$}
        (1,1.6) node[circle] (p5) {$5$}
        ;
        
        \draw
        (p1) edge[->] (p2)
        (p2) edge[->] (p1)
        
        (p2) edge[->] (p3)
        (p3) edge[->] (p2)
        
        (p3) edge[->] (p4)
        (p4) edge[->] (p3)
        
        (p1) edge[->] (p4)
        (p4) edge[->] (p1)
        
        (p2) edge[->] (p5)
        (p5) edge[->,bend right=60] (p4)
        
        (p4) edge[->] (p2)
        ;
      \end{scope}
    \end{tikzpicture}\hfill
    \begin{tikzpicture}[xscale=0.5,yscale=0.5,>=stealth]
      \tikzstyle{every  circle node}=[circle,draw,minimum size=4mm,inner sep=2pt]

      \begin{scope}[scale=1]
        \draw
        (1,-3) node {$\NOEVEN$}
        (0,0) node[circle,fill=lightgray] (p1) {$1$}
        (2,0) node[circle] (p2) {$2$}
        (2,-2) node[circle,fill=lightgray] (p3) {$3$}
        (0,-2) node[circle] (p4) {$4$}
        (1,1.6) node[circle] (p5) {$5$}
        ;
        
        \draw
        (p1) edge[->] (p2)
        (p2) edge[->] (p1)
        
        (p2) edge[->] (p3)
        (p3) edge[->] (p2)
        
        (p3) edge[->] (p4)
        (p4) edge[->] (p3)
        
        (p1) edge[->] (p4)
        (p4) edge[->] (p1)
        
        (p2) edge[->] (p5)
        (p5) edge[->,bend right=60] (p4)
        
        (p4) edge[->] (p2)
        ;
      \end{scope}
    \end{tikzpicture}\hfill
        \begin{tikzpicture}[xscale=0.5,yscale=0.5,>=stealth]
      \tikzstyle{every  circle node}=[circle,draw,minimum size=4mm,inner sep=2pt]

      \begin{scope}[scale=1]
        \draw
        (1,-3) node {$\BIP$}
        (0,0) node[circle] (p1) {$1$}
        (2,0) node[circle,fill=lightgray] (p2) {$2$}
        (2,-2) node[circle] (p3) {$3$}
        (0,-2) node[circle] (p4) {$4$}
        (1,1.6) node[circle] (p5) {$5$}
        ;
        
        \draw
        (p1) edge[->] (p2)
        (p2) edge[->] (p1)
        
        (p2) edge[->] (p3)
        (p3) edge[->] (p2)
        
        (p3) edge[->] (p4)
        (p4) edge[->] (p3)
        
        (p1) edge[->] (p4)
        (p4) edge[->] (p1)
        
        (p2) edge[->] (p5)
        (p5) edge[->,bend right=60] (p4)
        
        (p4) edge[->] (p2)
        ;
      \end{scope}
    \end{tikzpicture}\hfill
        \begin{tikzpicture}[xscale=0.5,yscale=0.5,>=stealth]
      \tikzstyle{every  circle node}=[circle,draw,,minimum size=4mm,inner sep=2pt]

      \begin{scope}[scale=1]
        \draw
        (1,-3) node {$\SYM$}
        (0,0) node[circle] (p1) {$1$}
        (2,0) node[circle,fill=lightgray] (p2) {$2$}
        (2,-2) node[circle] (p3) {$3$}
        (0,-2) node[circle] (p4) {$4$}
        (1,1.6) node[circle,fill=lightgray] (p5) {$5$}
        ;
        
        \draw
        (p1) edge[->] (p2)
        (p2) edge[->] (p1)
        
        (p2) edge[->] (p3)
        (p3) edge[->] (p2)
        
        (p3) edge[->] (p4)
        (p4) edge[->] (p3)
        
        (p1) edge[->] (p4)
        (p4) edge[->] (p1)
        
        (p2) edge[->] (p5)
        (p5) edge[->,bend right=60] (p4)
        
        (p4) edge[->] (p2)
        ;
      \end{scope}
    \end{tikzpicture}\longversion{\hspace{2cm}}
  \caption{Backdoors for the AF $F$ from Example~\ref{exa:af}, with
    respect to the indicated classes.}
  \label{fig:ex-af-bs}
\end{figure}
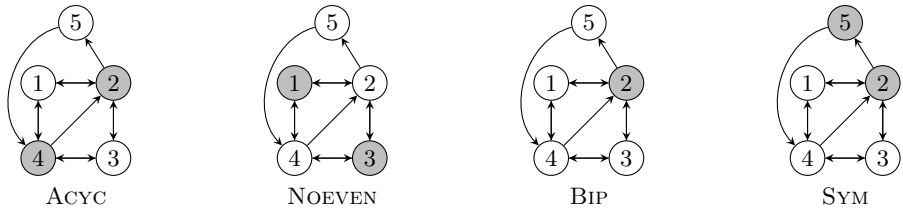
In the following we consider CA and SA parameterized by the distance to
a tractable fragment~$\CCC$.

\section{Tractability Results}

Regarding the fragments of acyclic and noeven argumentation frameworks we
obtain the following two results which show that these two fragments
admit an amplification.

\begin{THE}\label{the:fpt-k-acyc}
  The problems $\CA_\sigma$ and $\SA_\sigma$ are fixed-parameter
  tractable for parameter $\bsdist_\ACYC$ and the semantics $\sigma \in
  \{\ADM,\COM,\PRF,\SEM,\STB\}$.
\end{THE}

\begin{THE}\label{the:xp-k-noeven}
  The problems $\CA_\sigma$ and $\SA_\sigma$ are solvable in non-uniform
  polynomial-time for parameter $\bsdist_\NOEVEN$ and the semantics
  $\sigma \in \{\ADM,\COM,\PRF,\SEM,\STB\}$.
\end{THE}

\noindent The remainder of this section is devoted to a proof of
Theorems~\ref{the:fpt-k-acyc} and \ref{the:xp-k-noeven}.

The solution of the acceptance problems involves two tasks:
(i)~\emph{Backdoor Detection}: to find a $\CCC$\hy backdoor $B$ for $F$
of size at most $k$.  (ii)~\emph{Backdoor Evaluation}: to use the
$\CCC$\hy backdoor $B$ for $F$ for deciding whether $x$ is
credulously/skeptically accepted in $F$.

For backdoor detection we utilize recent results from fixed-parameter
algorithmics. For backdoor evaluation we introduce and use the new
concept of partial labelings.

\paragraph{Backdoor  Detection}
The following lemma gives an easy upper bound for the complexity of
detecting a $\CCC$\hy backdoor for any class $\CCC$ of AFs that can be
recognized in polynomial time.
\begin{PRO}\label{pro:xp-bsd}
  Let $\CCC$ be a class of AFs that can be recognized in polynomial time
  and $F=(X,A)$ an AF with $\bsdist_\CCC(F)\leq k$.  Then a
  $\CCC$-backdoor for $F$ of size at most~$k$ can be found in time
  $|X|^{O(k)}$ and hence in non-uniform polynomial-time for
  parameter~$k$.
\end{PRO}
\begin{proof}
  To find a $\CCC$-backdoor for $F$ of size at most $k$ we simply check
  for every subset $B \subseteq X$ of size $\leq k$ whether $F-B \in
  \CCC$.  There are $O(\Card{X}^{k})$ such sets and each check can be
  carried out in polynomial time.
\end{proof}
Together with Lemma~\ref{lem:nice} we obtain the following consequence of
Proposition~\ref{pro:xp-bsd}.
\begin{COR}
  Let $\CCC \in \{\ACYC,\NOEVEN,\SYM,\BIP\}$  and $F=(X,A)$ an AF with
  $\bsdist_\CCC(F)\leq k$.  Then a $\CCC$-backdoor for~$F$ of size at most
  $k$ can be found in time $|X|^{O(k)}$ and hence in non-uniform
  polynomial-time for parameter~$k$.
\end{COR}
It is a natural question to ask whether the above result can be improved
to uniform-polynomial time. We get an affirmative answer for three of
the four classes under consideration.
\begin{LEM}\label{lem:fpt-bsd}
  Let $\CCC \in \{\ACYC,\SYM,\BIP\}$ and $F=(X,A)$ an AF with
  $\bsdist_\CCC(F)\leq k$.  Then the detection of a $\CCC$-backdoor for~$F$ of
  size at most~$k$ is fixed-parameter tractable for parameter $k$.
\end{LEM}
\begin{proof}
  The detection of $\ACYC$-backdoors is easily seen to be equivalent
  to the so-called directed feedback vertex set problem which has
  recently been shown to be fixed-parameter tractable by
  \citex{ChenLiuLuOsullivanRazgon08}. Similarly, the detection of
  $\BIP$-backdoors is equivalent to the problem of finding an odd
  cycle traversal which is fixed-parameter tractable due to a result of
  \citex{ReedSmithVetta04}. Finally, the detection of a $\SYM$-backdoor
  set is equivalent to the vertex cover problem which is well known to
  be fixed-parameter tractable~\cite{DowneyFellows99}.
\end{proof}
We must leave it open whether the detection of $\NOEVEN$-backdoors
of size at most $k$ is fixed-parameter tractable for parameter $k$.  Since
already the polynomial-time recognition of $\NOEVEN$ is highly
nontrivial, a solution for the backdoor problem seems very
challenging.  However, it is easy to see that $\CCC$-backdoor
detection, considered as a non-parameterized problem, where $k$ is just
a part of the input, is $\NP$-complete for $\CCC \in
\{\ACYC,\NOEVEN,\SYM,\BIP\}$. Hence it is unlikely that
Lemma~\ref{lem:fpt-bsd} can be improved to a polynomial-time result.

\paragraph{Backdoor Evaluation}
Let $F=(X,A)$ be an AF. A \emph{partial labeling} of $F$, or
\emph{labeling} for short, is a function $\lambda: Y \rightarrow
\{\IN,\OUT,\UND\}$ defined on a subset $Y$ of $X$.  Partial labelings
generalize \emph{total} labelings which are defined on the entire set
$X$ of arguments \cite{ModgilCaminada09}.

We denote by $\IN(\lambda)$, $\OUT(\lambda)$ and $\UND(\lambda)$ the
sets of arguments $x \in X$ with $\lambda(x)=\IN$, $\lambda(x)=\OUT$ and
$\lambda(x)=\UND$ respectively. Furthermore, we set $\DEF(\lambda)=Y$
and $\UD(\lambda)=X \setminus \DEF(\lambda)$ and denote by
$\lambda_\emptyset$ the the \emph{empty labeling}, i.e., the labeling
with $\DEF(\lambda_\emptyset)=\emptyset$.  For a set $S \subseteq X$ we
define $\labFS{F}{S}$ to be the \emph{labeling of $F$ with respect to
  $S$} by setting $\IN(\labFS{F}{S})=S$, $\OUT(\labFS{F}{S})=S^+
\setminus S$ and $\UND(\labFS{F}{S})=X \setminus S^+$. We say a set $S
\subseteq X$ is \emph{compatible} with a labeling $\lambda$ if
$\lambda(x)=\labFS{F}{S}(x)$ for every $x \in \DEF(\lambda)$.





Let $F=(X,A)$ be an AF and
 $\lambda$ a partial labeling of~$F$. 
The \emph{propagation} of $\lambda$ with respect to
$F$, denoted $\lext{\lambda}$, is the labeling that is obtained from $\lambda$
by initially setting $\lext{\lambda}(x)=\lambda(x)$, for every $x \in
\DEF(\lambda)$, and subsequently applying one of the following three rules to 
unlabeled arguments $x \in X$ as long as possible.

\medskip \emph{Rule~1}.  $x$ is labeled $\OUT$ if $x$ has at least one
attacker that is labeled $\IN$.

\emph{Rule 2}.  $x$ is labeled $\IN$ if all attackers of $x$ are labeled
$\OUT$.

\emph{Rule 3}.  $x$ is labeled $\UND$ if all attackers of $x$ are
either labeled $\OUT$ or $\UND$ and at least one attacker of $x$ is
labeled $\UND$.  \medskip

\noindent It is easy to see that $\lext{\lambda}$ is well-defined and unique.

For an AF~$F$, a set $B$ of arguments of $F$ and a partial labeling
$\lambda$ of $F$ we set: 
\shortversion{
\smallskip
\noindent$
\begin{array}{lcl}
\COM^*(F,\lambda)&\!\!=\!\!&\SB \IN(\lext{\lambda}) \cup S \SM S
\in \ADM(F - \DEF(\lext{\lambda})) \SE;\\
\COM^*(F,B)&\!\!=\!\!&\bigcup_{\lambda : B \rightarrow
  \{\IN,\OUT,\UND\}}\COM^*(F,\lambda).
\end{array}$}

\longversion{
$$
\begin{array}{lcl}
\COM^*(F,\lambda)&\!\!=\!\!&\SB \IN(\lext{\lambda}) \cup S \SM S
\in \ADM(F - \DEF(\lext{\lambda})) \SE;\\[4pt]
\COM^*(F,B)&\!\!=\!\!&\bigcup_{\lambda : B \rightarrow
  \{\IN,\OUT,\UND\}}\COM^*(F,\lambda).
\end{array}$$}


\longversion{The following lemmas illustrate the connection between partial labelings and
complete extensions.}
\shortversion{The following lemma can be established by induction on the number of
arguments that have been labeled according to Rules 1--3.}
\begin{LEM}\label{prop:pe-le}
  Let $F=(X,A)$ be an AF, $\lambda$ a partial labeling of $F$, and $S$ a
  complete extension that is compatible with $\lambda$. Then the
  propagation $\lext{\lambda}$ of $\lambda$ is compatible with $S$.
\end{LEM}
\longversion{
\begin{proof} 
  We show the claim by induction on the number of arguments that have
  been labeled according to Rules 1--3. Because $S$ is compatible with
  $\lambda$ it holds that $\lext{\lambda}(x)=\lambda(x)=\labFS{F}{S}(x)$
  for every $x \in \DEF(\lambda)$ and hence the proposition holds before
  the first argument has been labeled according to one of the
  rules. Now, suppose that $\lambda'$ is the labeling that is obtained
  from $\lambda$ after labeling the first $i$ arguments according to one
  of the rules and that $x$ is the $i+1$-th argument that is labeled
  according to the rules. We distinguish three cases.

  First we assume that $x$ is labeled according to Rule~1.  In this case
  $\lext{\lambda}(x)=\OUT$ and we need to show that $x \in S^+ \setminus
  S$. It follows from the definition of Rule~1 that $x$ has at least one
  attacker $n$ with $\lambda'(n)=\IN$.  Using the induction hypothesis
  it follows that $\labFS{F}{S}(n)=\IN$ and hence $n \in S$. Because $S$
  is conflict-free it follows that $x \notin S$ but since $x$ is
  attacked by $n$ it follows that $x \in S^+ \setminus S$.

  Second we assume that $x$ is labeled according to Rule~2.  In this
  case $\lext{\lambda}(x)=\IN$ and we need to show that $x \in S$. Let
  $n_1,\dotso,n_r$ be all the attackers of $x$ in $F$. It follows from
  the definition of Rule~2 that $\lambda'(n_j)=\OUT$ for every $1 \leq j
  \leq r$. Using the induction hypothesis it follows that
  $\labFS{F}{S}(n_j)=\OUT$ and hence $n_j \in S^+ \setminus S$ for every
  $1 \leq j \leq r$. It follows that no out-neighbor of $x$ can be
  contained in~$S$ otherwise this out-neighbor would be attacked by~$x$
  but~$x$ cannot be defended by $S$. Hence $S \cup \{x\}$ is also
  admissible.  Because $x$ is defended by $S$ it follows that $x \in S$.

  Finally, we assume that $x$ is labeled according to Rule~3.  In this
  case $\lext{\lambda}(x)=\UND$ and we need to show that $x \notin
  S^+$. Using the definition of Rule~3 it follows that the set of all
  attackers of $x$ can be partitioned into two sets $U$ and $O$ such
  that $\lambda'(u)=\UND$ for every $u \in U$ and $\lambda'(o)=\OUT$ for
  every $o \in O$ and $U \neq \emptyset$. Using the induction hypothesis
  it follows that $\lambda'(n)=\labFS{F}{S}(n)$ for every $n \in U \cup
  O$. Hence, no attacker of $x$ belongs to $S$ and so $x$ cannot be
  contained in $S^+ \setminus S$. Furthermore, because $S$ is admissible
  and $x$ has an attacker that is not contained in $S^+$ it follows
  that $x$ cannot be contained in $S$. Hence, $x$ is not contained in
  $S^+$.
\end{proof}
}
\shortversion{The following lemma can be shown using Lemma~\ref{prop:pe-le}.}
\begin{LEM}\label{pro:ext-pe}
  Let $F=(X,A)$ be an AF and $B \subseteq X$.
  Then $\COM(F) \subseteq \COM^*(F,B)$.
\end{LEM}
\longversion{
\begin{proof}
  Let $F=(X,A)$ be the given AF, $B \subseteq X$ and $S \in \COM(F)$.
  We show that $S \in \COM^*(F,\lambda)=\SB \IN(\lext{\lambda}) \cup S \SM S
  \in \ADM(F-\DEF(\lext{\lambda})) \SE$ for the unique 
  partial labeling $\lambda$ defined on $B$ that is compatible with $S$.
  We set $S_1=S \cap \DEF(\lext{\lambda})$, $S_2=S \setminus S_1$,
  and $F_2=F - \DEF(\lext{\lambda})$.
  
  It follows from Lemma~\ref{prop:pe-le} that
  $S_1=\IN(\lext{\lambda})$. It remains to show that $S_2$ is admissible
  in $F_2$.  Clearly, $S_2$ is conflict-free. To see that $S_2$ is
  admissible suppose to the contrary that there is an argument $x \in
  S_2$ that is not defended by $S_2$ in $F_2$, i.e., $x$ has an attacker
  $y$ in $F_2$ that is not attacked by an argument in $S_2$. Because $S$
  is a complete extension of $F$ the argument $x$ is defended by $S$ in
  $F$. Hence, there is a $z \in S_1 = S_1=\IN(\lext{\lambda})$ that
  attacks $y$.  But then, using rule Rule~1, $\lext{\lambda}(y)=\OUT$,
  and hence $y$ cannot be an argument of $F_2$.  Hence $S_2$ is
  admissible in $F_2$.
\end{proof}
}

For an AF $F$ we set $\lext{F}=F-\DEF(\lext{\lambda_\emptyset})$. In other
words, $\lext{F}$ is obtained from $F$ after deleting all arguments from $F$
that, starting from the empty labeling, are labeled according to the
Rules~1--3. We observe that because we start from the empty labeling 
Rule~3 will not be invoked.

We say a class $\CCC$ of AFs is \emph{fully tractable} if (i)~for every
$F \in \CCC$ the set $\ADM(\lext{F})$ can be computed in polynomial
time, and (ii)~$\CCC$ is closed under the deletion of arguments, i.e.,
if $F=(X,A) \in \CCC$ and $Y\subseteq X$, then $F-Y \in \CCC$.

\begin{THE}\label{lem:bs-sig-fpt-acyc-noeven}
  Let $\CCC$ be a fully tractable class of AFs, $F=(X,A)$ an AF and $B$
  a $\CCC$\hy backdoor for $F$ with $\Card{B}\leq k$.  Then the
  computation of the sets $\COM(F)$, $\PRF(F)$, $\SEM(F)$ and $\STB(F)$
  can be carried out in time $3^{k}|X|^{O(1)}$ and is therefore
  fixed-parameter tractable for parameter~$k$.
\end{THE}
\begin{proof}
  Let $\CCC$ be a fully tractable class of AFs, $F=(X,A)$ an AF and $B$
  a $\CCC$\hy backdoor for $F$ with $\Card{B}\leq k$. We first show that
  the computation of $\COM(F)$ is fixed-parameter tractable for
  parameter~$k$. Let $\lambda$ be one of the $3^{k}$ partial labelings
  of $F$ defined on $B$. We first show that we can compute
  $\COM^*(F,\lambda)=\SB \IN(\lext{\lambda}) \cup S \SM S \in \ADM(F -
  \DEF(\lext{\lambda})) \SE$ in polynomial time, i.e., in time
  $\Card{X}^{O(1)}$. Clearly, we can compute the propagation
  $\lext{\lambda}$ of $\lambda$ in polynomial time. Furthermore, because
  $F-B \in \CCC$ ($B$ is a $\CCC$\hy backdoor) also $F-
  \DEF(\lext{\lambda})\in \CCC$; this follows since $B\subseteq
  \DEF(\lext{\lambda})$ and $\CCC$ is closed under argument deletion
  since $\CCC$ is assumed to be fully tractable.  Moreover, since $\CCC$
  is assumed to be fully tractable and $F- \DEF(\lext{\lambda})\in
  \CCC$, we can compute $\ADM(\lext{ (F - \DEF(\lext{\lambda}))})=\ADM(F
  - \DEF(\lext{\lambda}))$ in polynomial time.  Consequently, we can
  compute the set $\COM^*(F,\lambda)$ in polynomial time. Since there
  are at most $3^{k}$ partial labelings of $F$ defined on $B$, it
  follows that we can compute the entire set $\COM^*(F,B)$ in time
  $3^{k}|X|^{O(1)}$.

  By Lemma~\ref{pro:ext-pe} we have $\COM(F) \subseteq
  \COM^*(F,B)$. Thus we can obtain $\COM(F)$ from $\COM^*(F,B)$ by
  simply testing for each $S \in \COM^*(F,B)$ whether $S$ is a complete
  extension of $F$. It is a well-known fact that each such a test can be
  carried out in polynomial time (see
  e.g.~\cite{DvorakWoltran10b}). Hence, we conclude that indeed
  $\COM(F)$ can be computed in time $3^{k}|X|^{O(1)}$.

  For the remaining sets $\PRF(F)$, $\SEM(F)$ and $\STB(F)$ we note that
  each of them is a subset of $\COM(F)$. Furthermore, the extensions in
  $\PRF(F)$ are exactly the extensions in $\COM(F)$ which are maximal
  with respect to set inclusion. Similarly, the extensions in $\SEM(F)$
  are exactly the extensions in $S \in \COM(F)$ where the set $S^+$ is
  maximal with respect to set inclusion, and $\STB(F)$ are exactly the
  extensions $S \in \COM(F)$ where $S^+=X$. Clearly, these observations
  can be turned into an algorithm that computes from $\COM(F)$ the sets
  $\PRF(F),\SEM(F),\STB(F)$ in polynomial time.
\end{proof}


\begin{LEM}\label{lem:acyc-noeven-ft}
  The classes $\ACYC$ and $\NOEVEN$ are fully tractable.
\end{LEM}
\begin{proof}
  It is easy to see that both classes satisfy condition (ii) of being
  fully tractable, i.e., both classes are closed under the deletion of
  arguments. It remains to show that they also satisfy condition (i) of
  being fully tractable, i.e., for every $F \in \ACYC \cup
  \NOEVEN=\NOEVEN$ it holds that the set $\ADM(\lext{F})$ can be
  computed in polynomial time.  \citex{DunneBenchcapon01} have shown
  that if $F \in \NOEVEN$ and every argument of $F$ is contained in at
  least one directed cycle, then $\ADM(F)=\{\emptyset\}$. Consequently,
  it remains to show that if $F \in \NOEVEN$ then every argument of
  $\lext{F}$ lies on a directed cycle. To see this it suffices to show
  that every argument $x$ of $\lext{F}$ has at least one attacker in
  $\lext{F}$. Suppose not, i.e., there is an argument $x \in X \setminus
  \DEF(\lext{\lambda_\emptyset})$ with no attacker in $\lext{F}$. It
  follows that every attacker of $x$ must be labeled and hence $x \in
  \DEF(\lext{\lambda_\emptyset})$, a contradiction.
\end{proof}

\medskip\noindent Combining Theorem~\ref{lem:bs-sig-fpt-acyc-noeven}
with Lemma~\ref{lem:acyc-noeven-ft} we conclude that if $\CCC \in
\{\ACYC,\NOEVEN\}$ then the backdoor evaluation problem is
fixed-parameter tractable parameterized by the size of the backdoor set
for the semantics $\sigma \in \{\COM,\PRF,\SEM,\STB\}$. For the
remaining case of admissible semantics, we recall from
Table~\ref{tab:complexity} that $\SA_\ADM$ is trivial. Furthermore, we
observe that every admissible extension is contained in some complete
extension, and every complete extension is also admissible. We conclude
that an argument is credulously accepted with respect to the admissible
semantics if and only if the argument is credulously accepted with
respect to complete semantics. Hence, we have shown that 
backdoor evaluation is also fixed-parameter tractable with respect to
admissible semantics.  Together with Lemma~\ref{lem:fpt-bsd} and
Lemma~\ref{pro:xp-bsd} this establishes our main results
Theorem~\ref{the:fpt-k-acyc} and Theorem~\ref{the:xp-k-noeven} of this
section.

\begin{table}[bth]
  \shortversion{\small}
  \begin{center}
  \begin{tabular}{cc|ccc|cc}
    \multicolumn{2}{c}{$\lambda$} & \multicolumn{3}{c}{$\lext{\lambda}$} & &
    $\IN(\lext{\lambda}) \in $ \\
    $2$ & $4$ & $1$ & $3$ & $5$ & $\IN(\lext{\lambda})$ & $\COM(F)$? \\
    \hline
    \IN & \IN & \OUT & \OUT & \OUT & $\{2,4\}$ & no \\
    \IN & \OUT & \OUT & \OUT & \OUT & $\{2\}$ & no \\
    \IN & \UND & \OUT & \OUT & \OUT & $\{2\}$ & no \\
    \OUT & \IN & \OUT & \OUT & \IN & $\{4,5\}$ & no \\
    \OUT & \OUT & \IN & \IN & \IN & $\{1,3,5\}$ & yes \\
    \OUT & \UND & \UND & \UND & \IN & $\{5\}$ & no \\
    \UND & \IN & \OUT & \OUT & \UND & $\{4\}$ & no \\
    \UND & \OUT & \UND & \UND & \UND & $\emptyset$ & yes \\
    \UND & \UND & \UND & \UND & \UND & $\emptyset$ & yes \\
  \end{tabular}
\end{center}
\caption{Calculation of all complete extensions for the AF $F$ of
    Example~\ref{exa:af} using the $\ACYC$-backdoor $\{2,4\}$.}
  \label{tab:ext-lab}
\end{table}

\begin{EXA}
  Consider again the AF $F$ from Example~\ref{exa:af}. We have observed
  above that $F$ has an
 $\ACYC$-backdoor $B$ consisting of the
  arguments $2$ and $4$. We now show how to use the backdoor $B$ 
  to compute all complete extensions of $F$ using the procedure
  given in Theorem~\ref{lem:bs-sig-fpt-acyc-noeven}. Table~\ref{tab:ext-lab} shows the
  propagations for all partial labelings of $F$ defined on $B$ together with 
  the set $\IN(\lext{\lambda})$ and for every $\lambda$ it is indicated whether the set
  $\IN(\lext{\lambda})$ is a complete extension of $F$. Because $F-B$ is acyclic
  it follows that $\ADM(F-\DEF(\lext{\lambda}))=\{\emptyset\}$ (see the proof
  of Lemma~\ref{lem:acyc-noeven-ft}) and hence
  $\COM^*(F,\lambda)=\{\IN(\lext{\lambda})\}$. It is now easy to compute
  $\COM^*(F,B)$ as the union of all the sets $\IN(\lext{\lambda})$ given in
  Table~\ref{tab:ext-lab}. Furthermore, using the rightmost column of
  Table~\ref{tab:ext-lab} we conclude that $\COM(F)=\{\emptyset,\{1,3,5\}\}$,
  which is in full correspondence to our original observation in 
  Example~\ref{exa:af-ex}.\qed
\end{EXA}

\section{Hardness Results}
The hardness results for $\CA_\sigma$ and $\SA_\sigma$ are not
completely symmetric since for $\sigma \in \{\ADM, \COM\}$ the former
problem is $\NP$-complete, the latter is solvable in polynomial time
(recall Table~\ref{tab:complexity}).

\begin{THE}\label{the:ca-bip}
  (1)~The problem $\CA_\sigma$ is $\NP$-hard for AFs $F$ with
  $\bsdist_\BIP(F)=1$ and $\sigma \in \{\ADM,$ $\COM,$ $\PRF,$ $\SEM,$ $\STB\}$.
  (2)~The problem $\SA_\sigma$ is $\coNP$-hard for AFs $F$ with
  $\bsdist_\BIP(F)=1$ and $\sigma \in \{\PRF,\SEM,\STB\}$.
  \end{THE}
  \begin{proof} (Sketch.) The hardness results follow by reductions from
    Monotone $3$-Satisfiability~\cite{GareyJohnson79} and its
    complement, similar to  reductions used by \citex{Dunne07}. 
    \shortversion{Because of space restrictions we omit details and only
      illustrate the constructions} \longversion{We illustrate the
      constructions}
    in Fig.~\ref{fig:csa-bip}.
\end{proof}

\begin{figure}[tbh]
\centering\small
\longversion{\hspace{2cm}}
    \begin{tikzpicture}[xscale=0.6,yscale=.6,>=stealth]
      \tikzstyle{every node}=[circle,inner sep=3pt,draw]

      \begin{scope}[scale=1]
        \draw
        +(0,-1) node[label=left:$\phi$] (phi) {};
        
        \draw
        +(1,0) node[label=above:$C_1$] (C1) {}
        +(2,0) node[label=above:$\overline{x}_1$] (nx1) {}
        +(3,0) node[label=above:$\overline{x}_2$] (nx2) {}
        +(4,0) node[label=above:$\overline{x}_3$] (nx3) {}
        ;
        
        \draw
        +(1,-2) node[label=below:$\overline{C}_1$] (nC1) {}
        +(2,-2) node[label=below:$x_1$] (x1) {}
        +(3,-2) node[label=below:$x_2$] (x2) {}
        +(4,-2) node[label=below:$x_3$] (x3) {}
        ;

        \draw
        (x1) edge[->] (nx1)
        (nx1) edge[->] (x1)

        (x2) edge[->] (nx2)
        (nx2) edge[->] (x2)
        
        (x3) edge[->] (nx3)
        (nx3) edge[->] (x3)
        ;
        
        \draw
        (x1) edge[->] (C1)
        (x2) edge[->] (C1)
        (x3) edge[->] (C1);
        
        \draw
        (nx1) edge[->] (nC1)
        (nx2) edge[->] (nC1)
        (nx3) edge[->] (nC1);
        
        \draw
        (C1) edge[->] (phi)
        (nC1) edge[->] (phi);
      \end{scope}
    \end{tikzpicture}\hfill
    \begin{tikzpicture}[xscale=0.6,yscale=.6,>=stealth]
      \tikzstyle{every node}=[circle,inner sep=3pt,draw]
  
      \begin{scope}[scale=1]
        \draw 
        +(-1,-1) node[label=left:$\phi'$] (phis) {};

        \draw
        +(0,-1) node[label=above:$\phi$] (phi) {};
        
        \draw
        +(1,0) node[label=above:$C_1$] (C1) {}
        +(2,0) node[label=above:$\overline{x}_1$] (nx1) {}
        +(3,0) node[label=above:$\overline{x}_2$] (nx2) {}
        +(4,0) node[label=above:$\overline{x}_3$] (nx3) {}
        ;
        
        \draw
        +(1,-2) node[label=below:$\overline{C}_1$] (nC1) {}
        +(2,-2) node[label=below:$x_1$] (x1) {}
        +(3,-2) node[label=below:$x_2$] (x2) {}
        +(4,-2) node[label=below:$x_3$] (x3) {}
        ;

        \draw
        (x1) edge[->] (nx1)
        (nx1) edge[->] (x1)

        (x2) edge[->] (nx2)
        (nx2) edge[->] (x2)
        
        (x3) edge[->] (nx3)
        (nx3) edge[->] (x3)
        ;
        
        \draw
        (x1) edge[->] (C1)
        (x2) edge[->] (C1)
        (x3) edge[->] (C1);
        
        \draw
        (nx1) edge[->] (nC1)
        (nx2) edge[->] (nC1)
        (nx3) edge[->] (nC1);
        
        \draw
        (C1) edge[->] (phi)
        (nC1) edge[->] (phi);

        \draw
        (phi) edge[<-] (phis)
        (phi) edge[->] (phis);
      \end{scope}
    \end{tikzpicture}
\longversion{\hspace{2cm}}
  \caption{Illustrations for the reductions in the proof of
    Theorem~\ref{the:ca-bip}, showing instances $(F,\phi)$ and
    $(F',\phi')$ for the problems $\CA_\sigma$ and $\SA_\sigma$,
    respectively, obtained from the monotone 3-CNF formula $\phi=C_1
    \land \overline{C}_1$ with $C_1=x_1 \lor x_2 \lor x_3$ and
    $\overline{C}_1=\lnot x_1 \lor \lnot x_2 \lor \lnot x_3$. The set
    $\{\phi\}$ is a $\BIP$\hy backdoor for $F$ and $F'$.}
  \label{fig:csa-bip}
\end{figure}

\begin{THE}\label{the:ca-sym}
  (1)~The problem $\CA_\sigma$ is $\NP$-hard for AFs $F$ with
  $\bsdist_\SYM(F)=1$ and $\sigma \in \{\ADM,$ $\COM,$ $\PRF,$ $\SEM,$ $\STB\}$.
  (2)~ The problem $\SA_\sigma$ is $\coNP$-hard for AFs $F$ with
  $\bsdist_\SYM(F)=1$ and $\sigma \in \{\PRF,\SEM,\STB\}$.
\end{THE}
\begin{proof}(Sketch.)  We use a reduction from
  3-Satisfi\-ability~\cite{GareyJohnson79} and its complementary
  problem, similar to reductions used by \citex{DimopoulosTorres96}.
  We illustrate the reductions in Fig.~\ref{fig:csa-sym}.
\end{proof}

\begin{figure}[tbh]
    \vspace{-5mm}
\centering
\small\longversion{\hspace{2cm}}
    \begin{tikzpicture}[xscale=0.7,yscale=1,>=stealth]

      \tikzstyle{every node}=[circle,inner sep=3pt,draw]

      \begin{scope}[scale=1]
        \draw
        +(0,-0.3) node[label=above:$\phi$] (phi) {};
        
        \draw
        +(-2.5,-2) node[label=below:$x_1$] (x1) {}
        +(-1.5,-2) node[label=below:$\overline{x}_1$] (nx1) {}

        +(-0.5,-2) node[label=below:$x_2$] (x2) {}
        +(0.5,-2) node[label=below:$\overline{x}_2$] (nx2) {}

        +(1.5,-2) node[label=below:$x_3$] (x3) {}
        +(2.5,-2) node[label=below:$\overline{x}_3$] (nx3) {}
        ;
        
        \draw
        +(-1.75,-1) node[label=left:$C_1$] (C1) {}
        +(0,-1) node[label=left:$C_2$] (C2) {}
        +(1.75,-1) node[label=right:$C_3$] (C3) {}
        ;
        
        \draw
        (x1) edge[->] (nx1)
        (nx1) edge[->] (x1)

        (x2) edge[->] (nx2)
        (nx2) edge[->] (x2)

        (x3) edge[->] (nx3)
        (nx3) edge[->] (x3)
        ;
        
        \draw
        (x1) edge[->] (C1)
        (x2) edge[->] (C1)
        (x3) edge[->] (C1)

        (nx1) edge[->] (C2)
        (x2) edge[->] (C2)
        (nx3) edge[->] (C2)

        (nx1) edge[->] (C3)
        (nx2) edge[->] (C3)
        (nx3) edge[->] (C3)

        (x1) edge[<-] (C1)
        (x2) edge[<-] (C1)
        (x3) edge[<-] (C1)

        (nx1) edge[<-] (C2)
        (x2) edge[<-] (C2)
        (nx3) edge[<-] (C2)

        (nx1) edge[<-] (C3)
        (nx2) edge[<-] (C3)
        (nx3) edge[<-] (C3)
        ;

        \draw
        (C1) edge[->] (phi)
        (C2) edge[->] (phi)
        (C3) edge[->] (phi)
        ;
      \end{scope}
    \end{tikzpicture}\hfill 
    \begin{tikzpicture}[xscale=0.7,yscale=1,>=stealth]
      \tikzstyle{every node}=[circle,inner sep=3pt,draw]

      \begin{scope}[scale=1]
        \draw
        +(0,.3) node[label=above:$\phi'$] (phis) {};

        \draw
        +(0,-.3) node[label=left:$\phi$] (phi) {};
        
        \draw
        +(-2.5,-2) node[label=below:$x_1$] (x1) {}
        +(-1.5,-2) node[label=below:$\overline{x}_1$] (nx1) {}

        +(-0.5,-2) node[label=below:$x_2$] (x2) {}
        +(0.5,-2) node[label=below:$\overline{x}_2$] (nx2) {}

        +(1.5,-2) node[label=below:$x_3$] (x3) {}
        +(2.5,-2) node[label=below:$\overline{x}_3$] (nx3) {}
        ;
        
        \draw
        +(-1.75,-1) node[label=left:$C_1$] (C1) {}
        +(0,-1) node[label=left:$C_2$] (C2) {}
        +(1.75,-1) node[label=right:$C_3$] (C3) {}
        ;
        
        \draw
        (x1) edge[->] (nx1)
        (nx1) edge[->] (x1)

        (x2) edge[->] (nx2)
        (nx2) edge[->] (x2)

        (x3) edge[->] (nx3)
        (nx3) edge[->] (x3)
        ;
        
        \draw
        (x1) edge[->] (C1)
        (x2) edge[->] (C1)
        (x3) edge[->] (C1)

        (nx1) edge[->] (C2)
        (x2) edge[->] (C2)
        (nx3) edge[->] (C2)

        (nx1) edge[->] (C3)
        (nx2) edge[->] (C3)
        (nx3) edge[->] (C3)

        (x1) edge[<-] (C1)
        (x2) edge[<-] (C1)
        (x3) edge[<-] (C1)

        (nx1) edge[<-] (C2)
        (x2) edge[<-] (C2)
        (nx3) edge[<-] (C2)

        (nx1) edge[<-] (C3)
        (nx2) edge[<-] (C3)
        (nx3) edge[<-] (C3)
        ;

        \draw
        (C1) edge[->] (phi)
        (C2) edge[->] (phi)
        (C3) edge[->] (phi)
        ;

        \draw
        (phi) edge[->] (phis)
        (phis) edge[->] (phi)
        ;
      \end{scope}
    \end{tikzpicture}\longversion{\hspace{2cm}}

  \caption{Illustrations for the reductions in proof of
    Theorem~\ref{the:ca-sym}.  $(F,\phi)$ and $(F',\phi')$ are instances
    of $\CA_\sigma$ and $\SA_\sigma$, respectively, obtained from the
    3-CNF formula $\phi=C_1 \land C_2 \land C_3$ with $C_1=x_1 \lor x_2
    \lor x_3$, $C_2=\lnot x_1 \lor x_2 \lor \lnot x_3$ and $C_3=\lnot
    x_1 \lor \lnot x_2 \lor \lnot x_3$. The set
    $\{\phi\}$ is a $\SYM$\hy backdoor for $F$ and $F'$.}
  \label{fig:csa-sym}
\end{figure}

\section{Comparison with other Parameters}

In this section we compare our new structural parameters $\bsdist_\ACYC$
and $\bsdist_\NOEVEN$ to the parameters treewidth and clique-width that
have been introduced to the field of abstract argumentation by
\citex{Dunne07} and \citex{DvorakSzeiderWoltran10}, respectively. Due to
space requirements we cannot give the definitions of these parameters
and must refer the reader to the above references. The following two
propositions show that treewidth and clique-width are both incomparable
to our distance parameters.
\begin{PRO}
  There are acyclic and noeven AFs that have arbitrarily high treewidth
  and clique-width.
\end{PRO}
\begin{proof}
  Consider any symmetric AF $F$ of high treewidth or clique-width
  together with an arbitrary but fixed ordering $<$ of the arguments of
  $F$.  By deleting all attacks from an argument $x$ to an argument $y$
  with $y<x$ we obtain an acyclic AF $F'$ whose treewidth and
  clique-width is at least as high as the treewidth and clique-width of
  $F$, but $\bsdist_{\NOEVEN}(F)=\bsdist_{\ACYC}(F)=0$.
\end{proof}
\begin{PRO}
  There are AFs with bounded treewidth and clique-width where
  $\bsdist_{\NOEVEN}$ and $\bsdist_{\ACYC}$ are arbitrarily high.
\end{PRO}
\begin{proof}
  Consider the AF $F$ that consists of $n$ disjoint directed cycles of
  even length. It is easy to see that the treewidth and the clique-width
  of $F$ are bounded by a constant but
  $\bsdist_\NOEVEN(F)=\bsdist_{\ACYC}(F)=n$.
\end{proof}


\section{Conclusion}

We have introduced a novel approach to the efficient solution of
acceptance problems for abstract argumentation frameworks by
``augmenting'' a tractable fragment. This way the efficient solving
techniques known for a restricted fragment, like the fragment of acyclic
argumentation frameworks, become generally applicable to a wider range
of argumentation frameworks and thus relevant for real-world
instances. Our approach is orthogonal to decomposition-based approaches
and thus we can solve instances efficiently that are hard for known
methods. 

 
The augmentation approach entails two tasks, the detection of a small
backdoor and the evaluation of the backdoor. For the first task
we could utilize recent results from fixed-parameter algorithm design,
thus making results from a different research field applicable to
abstract argumentation. For the second task we have introduced the
concept of  partial labelings, which seems to us a useful tool that may
be of independent interest.
In view of the possibility of an augmentation, our results add
significance to known tractable fragments and  motivate the
identification of new tractable fragments.
For future research we plan to extend our results to other semantics and
new tractable fragments.


\begin{thebibliography}{}
\shortversion{\setlength{\itemsep}{0pt}}


\bibitem[\protect\citeauthoryear{Baroni and Giacomin}{2009}]{BaroniGiacomin09}
P. Baroni and M. Giacomin.
\newblock Semantics of abstract argument systems.
\newblock In I. Rahwan and G. Simari, eds, {\em Argumentation in
  Artificial Intelligence}, pages 25--44. Springer Verlag, 2009.

\bibitem[\protect\citeauthoryear{Bench-Capon and
  Dunne}{2007}]{BenchcaponDunne07}
T.~J.~M. Bench-Capon and P.~E. Dunne.
\newblock Argumentation in artificial intelligence.
\newblock {\em Artificial Intelligence}, 171(10-15):619--641, 2007.

\bibitem[\protect\citeauthoryear{Besnard and Hunter}{2008}]{BesnardHunter08}
P. Besnard and A. Hunter.
\newblock {\em Elements of Argumentation}.
\newblock The MIT Press, 2008.

\bibitem[\protect\citeauthoryear{Bondarenko \bgroup \em et al.\egroup
  }{1997}]{BondarenkoDungKowalskiToni97}
A.~Bondarenko, P.~M. Dung, R.~A. Kowalski, and F.~Toni.
\newblock An abstract, argumentation-theoretic approach to default reasoning.
\newblock {\em Artificial Intelligence}, 93(1-2):63--101, 1997.


\bibitem[\protect\citeauthoryear{Chen \bgroup \em et al.\egroup
  }{2008}]{ChenLiuLuOsullivanRazgon08}
J. Chen, Y. Liu, S. Lu, B. O'Sullivan, and I. Razgon.
\newblock A fixed-parameter algorithm for the directed feedback vertex set
  problem.
\newblock {\em J. ACM}, 55(5):Art. 21, 19, 2008.

\bibitem[\protect\citeauthoryear{Coste-Marquis \bgroup \em et al.\egroup
  }{2005}]{CostemarquisDevredMarquis05}
S. Coste-Marquis, C. Devred, and P. Marquis.
\newblock Symmetric argumentation frameworks.
\newblock In L. Godo, ed., {\em ECSQARU 2005}, {\em LNCS} 3571,  317--328. Springer, 2005.

\bibitem[\protect\citeauthoryear{Dimopoulos and
  Torres}{1996}]{DimopoulosTorres96}
Y. Dimopoulos and A. Torres.
\newblock Graph theoretical structures in logic programs and default theories.
\newblock {\em Theoret. Comput. Sci.}, 170(1-2):209--244, 1996.

\bibitem[\protect\citeauthoryear{Downey and Fellows}{1999}]{DowneyFellows99}
R.~G. Downey and M.~R. Fellows.
\newblock {\em Parameterized Complexity}.
\newblock  Springer, 1999.

\bibitem[\protect\citeauthoryear{Dung}{1995}]{Dung95}
Ph.~M. Dung.
\newblock On the acceptability of arguments and its fundamental role in
  nonmonotonic reasoning, logic programming and {$n$}-person games.
\newblock {\em Artificial Intelligence}, 77(2):321--357, 1995.

\bibitem[\protect\citeauthoryear{Dunne and
  Bench-Capon}{2001}]{DunneBenchcapon01}
P.~E. Dunne and T.~J.~M. Bench-Capon.
\newblock Complexity and combinatorial properties of argument systems.
\newblock Technical report, University of Liverpool, 2001.

\bibitem[\protect\citeauthoryear{Dunne and
  Bench-Capon}{2002}]{DunneBenchcapon02}
P.~E. Dunne and T.~J.~M. Bench-Capon.
\newblock Coherence in finite argument systems.
\newblock {\em Artificial Intelligence}, 141(1-2):187--203, 2002.

\bibitem[\protect\citeauthoryear{Dunne}{2007}]{Dunne07}
P.~E. Dunne.
\newblock Computational properties of argument systems satisfying
  graph-theoretic constraints.
\newblock {\em Artificial Intelligence}, 171(10-15):701--729, 2007.

\bibitem[\protect\citeauthoryear{{Dvo\v r\'ak} and
  Woltran}{2010}]{DvorakWoltran10b}
W. {Dvo\v r\'ak} and S. Woltran.
\newblock On the intertranslatability of argumentation semantics.
\newblock In {\em Conference on Thirty Years of
  Nonmonotonic Reasoning (NonMon@30)}, Lexington, KY, USA, 2010.

\bibitem[\protect\citeauthoryear{Dvo\v{r}\'{a}k \bgroup \em et al.\egroup
  }{2010}]{DvorakSzeiderWoltran10}
W. Dvo\v{r}\'{a}k, S. Szeider, and S. Woltran.
\newblock Reasoning in argumentation frameworks of bounded clique-width.
\newblock In P. Baroni, F. Cerutti, M. Giacomin, and
  G.~R. Simari, eds., {\em COMMA 2010}, volume 216 of 
 {\em Frontiers in Artificial
  Intelligence and Applications}, pages 219--230. IOS, 2010.

\bibitem[\protect\citeauthoryear{Fichte and Szeider}{2011}]{FichteSzeider11}
J.~K. Fichte and S. Szeider.
\newblock Backdoors to tractable answer-set programming.
\newblock In {\em IJCAI 2011}, 2011.

\bibitem[\protect\citeauthoryear{Flum and Grohe}{2006}]{FlumGrohe06}
J. Flum and M. Grohe.
\newblock {\em Parameterized Complexity Theory}, volume XIV of {\em Texts in
  Theoretical Computer Science. An EATCS Series}.
\newblock Springer Verlag, Berlin, 2006.

\bibitem[\protect\citeauthoryear{Garey and Johnson}{1979}]{GareyJohnson79}
M.~R. Garey and D.~R. Johnson.
\newblock {\em Computers and Intractability}.
\newblock W. H. Freeman and Company, New York, San Francisco, 1979.

\bibitem[\protect\citeauthoryear{Gottlob and Szeider}{2006}]{GottlobSzeider08}
G. Gottlob and S. Szeider.
\newblock Fixed-parameter algorithms for artificial intelligence, constraint
  satisfaction, and database problems.
\newblock {\em The Computer Journal}, 51(3):303--325, 2006.




\bibitem[\protect\citeauthoryear{Modgil and Caminada}{2009}]{ModgilCaminada09}
S. Modgil and M. Caminada.
\newblock Proof theories and algorithms for abstract argumentation frameworks.
\newblock In I. Rahwan and G. Simari, eds., {\em Argumentation in
  Artificial Intelligence}, pages 105--132. Springer, 2009.



\bibitem[\protect\citeauthoryear{Niedermeier}{2006}]{Niedermeier06}
R. Niedermeier.
\newblock {\em Invitation to Fixed-Parameter Algorithms}.
\newblock Oxford Lecture Series in Mathematics and its Applications. Oxford
  University Press, Oxford, 2006.

\bibitem[\protect\citeauthoryear{Parsons \bgroup \em et al.\egroup
  }{2003}]{ParsonsWooldridgeAmgoud03}
S. Parsons, M. Wooldridge, and L. Amgoud.
\newblock Properties and complexity of some formal inter-agent dialogues.
\newblock {\em J. Logic Comput.}, 13(3):347--376, 2003.

\bibitem[\protect\citeauthoryear{Pollock}{1992}]{Pollock92}
J.~L. Pollock.
\newblock How to reason defeasibly.
\newblock {\em Artificial Intelligence}, 57(1):1--42, 1992.

\bibitem[\protect\citeauthoryear{Rahwan and Simari}{2009}]{RahwanSimari09}
I. Rahwan and G.~R. Simari, eds.
\newblock {\em Argumentation in Artificial Intelligence}.
\newblock Springer, 2009.

\bibitem[\protect\citeauthoryear{Reed \bgroup \em et al.\egroup
  }{2004}]{ReedSmithVetta04}
B. Reed, K. Smith, and A. Vetta.
\newblock Finding odd cycle transversals.
\newblock {\em Oper. Res. Lett.}, 32(4):299--301, 2004.

\bibitem[\protect\citeauthoryear{Robertson \bgroup \em et al.\egroup
  }{1999}]{RobertsonSeymourThomas99}
N. Robertson, P.~D. Seymour, and R. Thomas.
\newblock Permanents, {P}faffian orientations, and even directed circuits.
\newblock {\em Ann. of Math. (2)}, 150(3):929--975, 1999.


\bibitem[\protect\citeauthoryear{Samer and Szeider}{2009}]{SamerSzeider08c}
M. Samer and S. Szeider.
\newblock Fixed-parameter tractability.
\newblock In A. Biere, M. Heule, H. van Maaren, and T. Walsh,
  editors, {\em Handbook of Satisfiability}, chapter~13, pages 425--454. IOS
  Press, 2009.


\bibitem[\protect\citeauthoryear{Samer and Szeider}{2009a}]{SamerSzeider09a}
M. Samer and S. Szeider.
\newblock Backdoor sets of quantified {B}oolean formulas.
\newblock {\em Journal of Automated Reasoning}, 42(1):77--97, 2009.


\bibitem[\protect\citeauthoryear{Williams \bgroup \em et al.\egroup
  }{2003}]{WilliamsGomesSelman03}
R. Williams, C. Gomes, and B. Selman.
\newblock Backdoors to typical case complexity.
\newblock In G. Gottlob and T. Walsh, eds, {\em  IJCAI
  2003}, pages 1173--1178, 2003.

\end{thebibliography}

{ \shortversion{\small }

}
 
\end{document}